\documentclass[10pt,conference,a4paper]{IEEEtran}
\usepackage{algorithm} 
\usepackage{algorithmic}  
\usepackage[algo2e,linesnumbered]{algorithm2e} 
\usepackage{amsfonts}
\usepackage{amsbsy}
\usepackage{graphicx}
\usepackage{physics}
\usepackage{subfigure}
\usepackage{xcolor}
\usepackage{amsmath}
\usepackage{amssymb}
\usepackage{amsthm}
\usepackage{url}

\newtheorem{theorem}{Theorem}
\newtheorem{proposition}{Proposition}
\newtheorem{definition}{Definition}
\newcommand{\X}{\mathcal{X}}
\newcommand{\Y}{\mathcal{Y}}
\newcommand{\F}{\mathcal{F}}
\newcommand{\Set}{\mathcal{S}}
\newcommand{\Loss}{\mathcal{L}}
\newcommand{\R}{\mathcal{R}}
\newcommand{\K}{\mathfrak{K}}
\newcommand{\x}{\mathrm{x}}
\newcommand*\Bell{\ensuremath{\boldsymbol\ell}}

\begin{document}
\title{MMGAN: Manifold-Matching Generative Adversarial Network}

% \author{Noseong Park$^{\dagger,\ddagger}$ \\ University of North Carolina\\ Charlotte, USA \\ npark2@uncc.edu \And
% Ankesh Anand$^{\dagger}$ \\ Montreal Institute for Learning Algorithms\\ Montreal, Canada \\ankesh.anand@umontreal.ca  \\ \And
% Joel Ruben Antony Moniz$^{\dagger}$ \\ Carnegie Mellon University\\ Pittsburgh, USA \\ jrmoniz@andrew.cmu.edu \AND
% Kookjin Lee \\ University of Maryland\\ College Park, USA \\ klee@cs.umd.edu \\ \And
% Tanmoy Chakraborty \\ Indraprastha Institute of Information Technology\\ Delhi, India \\ tanmoy@iiitd.ac.in \And
% Jaegul Choo \\ Korea University\\ Seoul, Korea \\  jchoo@korea.ac.kr \\ \AND
% Hongkyu Park and Youngmin Kim \\ Electronics and Telecommunications Research Institute\\ Daejeon, Korea \\ \{hkpark,injesus\}@etri.re.kr
% }

\author{\IEEEauthorblockN{Noseong Park}
\IEEEauthorblockA{University of North Carolina\\ Charlotte, USA \\ npark2@uncc.edu}
\and
\IEEEauthorblockN{Ankesh Anand}
\IEEEauthorblockA{Montreal Institute for Learning Algorithms\\ Montreal, Canada \\ankesh.anand@umontreal.ca}
\and
\IEEEauthorblockN{Joel Ruben Antony Moniz}
\IEEEauthorblockA{Carnegie Mellon University\\ Pittsburgh, USA \\ jrmoniz@andrew.cmu.edu}
\and
\IEEEauthorblockN{Kookjin Lee}
\IEEEauthorblockA{University of Maryland\\ College Park, USA \\ klee@cs.umd.edu}
\and
\IEEEauthorblockN{Jaegul Choo}
\IEEEauthorblockA{Korea University\\ Seoul, Korea \\  jchoo@korea.ac.kr }
\and
\IEEEauthorblockN{Tanmoy Chakraborty}
\IEEEauthorblockA{IIIT Delhi, India \\ tanmoy@iiitd.ac.in}
\and
\IEEEauthorblockN{Hongkyu Park and Youngmin Kim}
\IEEEauthorblockA{Electronics and Telecommunications Research Institute\\ Daejeon, Korea \\ \{hkpark,injesus\}@etri.re.kr}
}


%Telephone: (800) 555--1212, Fax: (888) 555--1212}


\maketitle

\begin{abstract}
% Generative adversarial networks (GANs) are considered a new overarching paradigm in the world of generative models.
It is well-known that GANs are difficult to train, and several different techniques have been proposed in order to stabilize their training. In this paper, we propose a novel training method called \textit{manifold-matching}, and a new GAN model called \textit{manifold-matching GAN} (MMGAN). MMGAN finds two manifolds representing the vector representations of real and fake images. If these two manifolds match, it means that real and fake images are statistically identical. To assist the manifold-matching task, we also use i) kernel tricks to find better manifold structures, ii) moving-averaged manifolds across mini-batches, and iii) a regularizer based on correlation matrix to suppress mode collapse.

We conduct in-depth experiments with three image datasets and compare with several state-of-the-art GAN models. 32.4\% of images generated by the proposed MMGAN are recognized as fake images during our user study (16\% enhancement compared to other state-of-the-art model). MMGAN achieved an unsupervised inception score of 7.8 for CIFAR-10.  
\end{abstract}


% \fi
\IEEEpeerreviewmaketitle

\section{Introduction}
%\blfootnote{Equal contribution$^{\dagger}$, Corresponding author$^{\ddagger}$}
% Generating images can be used in various computer vision applications, and there exist several different types of generative models~\cite{journals/corr/KingmaW13,goodfellow2014generative}.
\textit{Generative adversarial networks} (GANs)~\cite{goodfellow2014generative} have been recently proposed to generate synthetic images.
% GANs consist of two different neural network models -- a generator $G$ and a discriminator $D$. They perform a specially designed zero-sum minimax game, where the discriminator $D$ tries to differentiate samples generated by the generator $G$ from real samples, and the generator $G$ tries to obfuscate the task of the discriminator $D$ by generating realistic samples. The primary motivation behind GANs was to feed the discriminator's classification results back to the generator to improve the generation process. In the overall framework of GANs, this feedback mechanism can be efficiently implemented through backpropagation.
However, it had been shown that it is notoriously difficult to train GANs due to several reasons~\cite{DBLP:journals/corr/ArjovskyB17}. One of the main difficulties arise by ill-designed loss functions and zero gradients.

Researchers have proposed several variations to solve such a difficult training process~\cite{hjelm2017boundary,arjovsky2017wasserstein,DBLP:journals/corr/SalimansGZCRC16}. In this paper, we propose a new loss function based on \emph{manifold-matching}. The manifold hypothesis puts forth the proposition that natural data will form manifolds in the embedding space~\cite{NIPS2010_3958}. From this perspective, it is likely that there exist two distinct manifolds: one formed by real data samples and another formed by generated samples. Our proposed method trains the generator $G$ with a specially designed loss function to match the two manifolds. Particularly in this paper, we assume a spherical manifold for various computational conveniences that it affords, such as kernel tricks. Furthermore, a spherical manifold can be described by a center point and its radius, making it easy to check if two spherical manifolds match. We can easily apply kernel tricks to spherical manifolds~\cite{hofmann2008}, by representing a space mapping via a kernel, thereby being able to apply the mapping without explicitly performing the mapping operation. Thus, our definition of manifold-matching accompanies space mappings (i.e., kernel tricks) to find more reliable manifold representations of images. This approach is better than using manifolds in the original feature space. The proposed manifold-matching GAN (MMGAN) is named after this notion of manifold-matching.

Training with mini-batches is the standard way to train neural networks. We use the \textit{exponentially weighted moving average} to find center points and radii of two manifolds across multiple mini-batches and match them. This is more stable than the local manifold-matching in a mini-batch.
% For two spherical manifolds representing real and generate image sets, we use \textit{exponentially weighted moving average} to calculate their \emph{global} center points and radii across all mini-batches. This provides more stable training than the manifold-matching in local mini-batches.

We also propose a regularization term to enhance diversity in generated samples, which allows MMGAN to avoid mode collapsing, a well-known problem in GANs. That is, the manifold of generated samples will collapse to a point if all the samples are too similar.
%, which makes the diversity of the samples generated as a key factor of the proposed MMGAN. This diversity is also required in the generation step, after training has been completed.\jaegul{Could you clarify this part? I don't understand what this means. }

We implement the proposed MMGAN on top of the deep convolutional GAN (DCGAN)~\cite{radford2015unsupervised} and improved GAN (IGAN) ~\cite{DBLP:journals/corr/SalimansGZCRC16}. DCGAN and IGAN have similar neural network architectures, but IGAN uses a more advanced training method than DCGAN. We replace the loss functions of the two state-of-the-art GAN models with our proposed manifold-matching loss.

To evaluate the performance, we use three popular image datasets: MNIST~\cite{lecun-mnisthandwrittendigit-2010}, CelebA~\cite{liu2015faceattributes}, and CIFAR-10~\cite{cifar}. Our experiments shows that MMGAN performs as good as popular GAN models or sometimes even better than them.

\section{Related Work and Preliminaries}
\subsection{Generative Adversarial Networks (GANs)}
The following minimax game equation describes the core idea of GANs~\cite{goodfellow2014generative}. Two players, a discriminator $D$ and a generator $G$, in the zero-sum minimax game are alternately trained by the following objective:
\begin{equation}\label{eq:gan}\begin{aligned}
\min_{G} \max_{D} V(G,D) =  & \mathbb{E}[\log D(x)]_{x \sim p_{data}(x)} \\
  & + \mathbb{E}[\log (1-D(G(z)))]_{z \sim p(z)},
\end{aligned}\end{equation}
where $p(z)$ is a prior distribution, $G(z)$ is a generator function, and $D(\cdot)$ is a discriminator function whose output spans $[0,1]$. $D(x)=0$ (resp. $D(x)=1$) indicates that the discriminator $D$ classifies a sample $x$ as \textit{generated} (resp. \textit{real}).

% Algorithm~\ref{alg:gan} shows the general training concept of GANs. $G$ and $D$ can be any form of neural networks. The discriminator $D$ tries to maximize the objective, whereas the generator $G$ tries to minimize it. In other words, the discriminator $D$ tries to distinguish between real and generated samples; the generator $G$ tries to generate realistic fake samples that the discriminator $D$ cannot distinguish from real samples.

% \begin{algorithm2e}[!t]
% \footnotesize
% \DontPrintSemicolon
% \hrule
% \KwIn{Real Samples: $\{x_1, x_2, \cdots\} \sim p(x)$}
% \KwOut{a Generative Model $G$}
% \hrule
% %{\color{blue}//Initialize a generator $G$ and discriminator $D$}\;
% $G \gets$ a generative neural network\;
% $D \gets$ a discriminator neural network\;
% \While{until converge} {
% Create a mini-batch of real samples $X = \{x_1,\cdots,x_n\}$\;
% Create a set of generated samples $Z = \{z_1,\cdots,z_n\}$\;
% Train the discriminator $D$ by maximizing Equation~\eqref{eq:gan}\;\label{alg:dis}
% Train the generator $G$ by minimizing Equation~\eqref{eq:gan};
% }
% \Return $G$\;
% \hrule
% \caption{\strut Training algorithm of GANs\label{alg:gan}}
% \end{algorithm2e}

Previously, it was shown that it is possible that a manifold $M_R$ representing real data samples transversely intersects with another manifold $M_G$ representing samples generated by the generator $G$ during the training process~\cite{DBLP:journals/corr/ArjovskyB17}. This describes the main intuition behind why the discriminator $D$ becomes really strong (because it is relatively easy to build such a classifier that clearly distinguishes two transversely intersecting manifolds in an ambient space), which deteriorates the entire learning process of GANs.

\subsection{Other GAN Models}

In Wasserstein GANs (WGANs)~\cite{arjovsky2017wasserstein}, the distance between the target distribution $p_{data}(x)$ and the generated sample distribution $p_{g}(x)$ is used in the loss function. When the loss function is minimized, the two distributions become identical (e.g., $p_{data}(x) = p_g(x)$). However, WGANs are known to be unstable if the gradients of the loss function are large. Therefore, they clip weights if they are too large after each stochastic gradient descent update.

In~\cite{DBLP:journals/corr/BellemareDDMLHM17,DBLP:journals/corr/MrouehSG17}, new distance metrics were proposed. The popular Wasserstein metric is a special case of the integral probability metric (IPM) of~\cite{DBLP:journals/corr/MrouehSG17} and the energy distance of ~\cite{DBLP:journals/corr/BellemareDDMLHM17} shows \textit{unbiased sum gradients}. Both metrics show no less performance than the Wasserstein distance.

IGANs~\cite{DBLP:journals/corr/SalimansGZCRC16} proposed several heuristics to better train GANs: one is \emph{feature matching}, and the other is \emph{mini-batch discrimination}. In particular, IGANs is considered the best generative model for various image datasets~\cite{DBLP:journals/corr/Goodfellow17}, and it has achieved a high inception score (see Section~\ref{sec:env} for its description) for CIFAR-10~\cite{DBLP:journals/corr/SalimansGZCRC16}. CVAE-GAN~\cite{DBLP:journals/corr/BaoCWLH17}, EBAGN~\cite{2016arXiv160903126Z}, and some others~\cite{2016arXiv160603439K} extended the concept of feature matching of IGAN to pairwise feature matching and some other different feature matching concepts. These approaches are similar to our method. However, our MMGAN considers not only features but also many other factors such as manifold, kernel trick, moving average, and correlation regularization.

On the other hand, auto-encoders have been integrated with GANs. In adversarial generator-encoder networks~\cite{DBLP:journals/corr/UlyanovVL17a}, the authors proposed an adversarial architecture between an encoder and a generator, in order to match the data distributions of real and fake images. They showed that the encoder-generator game is advantageous in matching distributions. $\alpha$-GAN is one of the most recent auto-encoding GANs. It uses variational inference for training, where the intractable likelihood function is replaced by a synthetic likelihood and the unknown posterior distribution is replaced by an implicit function. Afterwards, the variational auto-encoder and the generative adversarial network can be successfully merged.

There exist many other GAN architectures. However, we primarily compare our proposed model with IGAN, DCGAN, and auto-encoding GANs, considering their popularity and influence on numerous GAN variants. For image generation tasks, most of the GAN models use the discriminator and generator neural network architectures proposed by DCGAN, and among them, IGAN is considered as the state-of-the-art for several image datasets~\cite{DBLP:journals/corr/Goodfellow17}.

\section{Manifold-Matching GAN}
In this section, we first describe the main idea of \emph{manifold-matching}, and then the proposed manifold-matching GAN (MMGAN).

\subsection{Manifold-Matching}\label{sec:mm}
\begin{figure}[t]
\centering
\includegraphics[width=0.2\textwidth]{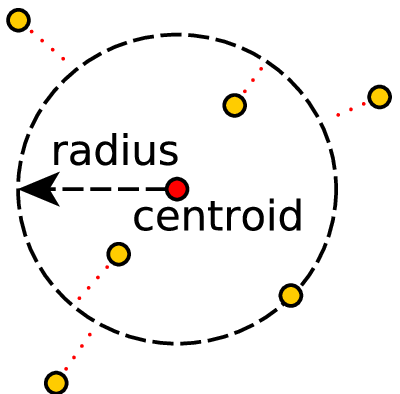}
\caption{Yellow points represent vector representations of data samples in a feature space. A manifold is represented by a sphere in the same feature space when $\varphi(s) = s$. A red point represents the centroid of the yellow points. The best manifold representing yellow points should minimize the sum of errors highlighted in red dotted lines.}\label{fig:manifold}
\end{figure}

Let $\Y = \{\mathbf{y_1}, \mathbf{y_2}, \cdots, \mathbf{y_n}\}$ (resp. $\X = \{\mathbf{x_1}, \mathbf{x_2}, \cdots, \mathbf{x_n}\}$) be the set of $d$-dimensional vector representations of real samples (resp. generated samples) --- we use bold font to denote vectors, and thus, $\mathbf{x}$ is a vector representation of a sample $x$. Note that we can extract these vector representations from the last layer of the discriminator.

Given a set $\Set$ of $n$ points or vectors (e.g., $\Set = \{\mathbf{s_1}, \mathbf{s_2}, \cdots, \mathbf{s_n}\}$), let $M_S$ be a manifold representing these points. For instance, $n$ points can be described by a sphere of centroid $\mathbf{c}$ and radius $\ell$ as follows:
\begin{equation}
\| \varphi(\mathbf{c}) - \varphi(\mathbf{s_i}) \|^2_2 = \ell^2,
\end{equation} where $\mathbf{s_i} \in \Set$ and $\varphi(\cdot)$ is a space mapping function. 
% In the case of bounding box, it can be described by a center point $\mathbf{c}$ and a dimension-specific radius $\ell_j$ as follows:

% \begin{equation}
% | c_j - s_{i,j} | = \ell_j\textrm{ in each dimension $j$},
% \end{equation}where $\mathbf{s_i} \in \Set$ and $x_j$ means the $j$-th value of a vector $\mathbf{x}$.
\textbf{Note that our manifold definition is not a simple sphere in the original feature space but a distorted sphere after the space mapping}. The key point is which space mapping we will use. It is not easy and might take a lot of time to find the best manifold that describes the set $\Set$ without any errors through a mapping $\varphi$. To this end, existing methods such as manifold learning~\cite{Roweis00nonlineardimensionality} can be adopted. For instance, a projection matrix that projects vector representations onto a low-dimensional sphere and preserves distances in the original space within local neighborhoods can be learned. However, we do not adopt this approach because it will increase the training time. Instead, we use several popular kernels equivalent to space mappings. This kernel trick enables us to minimize computational overheads.

For the sake of simplicity, in this section, let us assume the simplest mapping function $\varphi(x) = x$. We describe other mappings in Section~\ref{sec:kernel}.

With $\varphi(s) = s$, the best manifold is achieved when its center is the centroid $\mathbf{c}$ of $n$ points, and the radius is the mean distance from the centroid to $n$ points. This manifold minimizes the error, as shown in Figure~\ref{fig:manifold}, where the error is defined as the minimum distance from the points to the manifold (i.e., in this case, to the surface of a sphere).

\begin{proposition}
Given $n$ points in $d$-dimensional space, let us define its centroid $\mathbf{c}$ as 
\begin{equation}
\mathbf{c} = \frac{\sum_i{s_i}}{n}.
\end{equation}
Then, the centroid $\mathbf{c}$ minimizes the sum of squared distances to all the points, which is a well-known property.

\end{proposition}

\begin{proposition}
Given $n$ points and their centroid $\mathbf{c}$, let us define the radius $\ell$ as
\begin{equation}
\ell = \frac{ \|\mathbf{c} - \mathbf{s_i} \|_2 }{n}.
\end{equation}
Then, the radius $\ell$ minimizes the sum of squared errors from the manifold to all the points.
\end{proposition}

manifold-matching involves checking how similar $M_\Y$ and $M_\X$ are to each other. If these two manifolds are identical, we can effectively say that $p_{data} = p_{g}$ from the perspective of the discriminator (because the vector representations are obtained from the last layer of the discriminator). We use the following criteria to check if the two manifolds match.

\begin{definition}[manifold-matching Condition]\label{def:mm}
Two sphere manifolds $M_\Y$ and $M_\X$ are identical if their centroids and radii are the same. For sphere manifolds, we check if $\| \mathbf{c_{M_\Y}} - \mathbf{c_{M_\X}} \|_2 = 0$ and $|\ell_{M_\Y} - \ell_{M_\X} |= 0$.
% and for box manifolds, we first define the following $d$-dimensional vector. $$\Bell = <\ell_{M_\Y,0} - \ell_{M_\X,0}, \ell_{M_\Y,1} - \ell_{M_\X,1}, \cdots, \ell_{M_\Y,d} - \ell_{M_\X,d}>$$
% If $\|\Bell\| = 0$, then two box manifolds have the same volume.
\end{definition}

% Given a set $\Set = \{\mathbf{s_1}, \mathbf{s_2}, \cdots, \mathbf{s_n}\}$, we can quickly find the circumcenter $\mathbf{c}$ and radius $\ell$ or $\Bell$ of the manifold $M_\Set$ as follows:
%  \begin{align}
%  \ell\textrm{ or }\Bell = &
%  \begin{cases}
%  \max_i{ \|\mathbf{c} - \mathbf{s_i} \|_2  }, & \textrm{if $M_\Set$ is a sphere}\\
%  <max_0-min_0, max_1-min_1, \cdots, max_d-min_d>, & \textrm{if $M_\Set$ is a box}\\
%  \end{cases}\\
%  \mathbf{c} &=
%  \begin{cases}
%  \frac{\sum_i{s_i}}{n}, & \textrm{if $M_\Set$ is a sphere}\\
%  \frac{\Bell}{2} + <min_0, min_1, \cdots, min_d>, & \textrm{if $M_\Set$ is a box}\\
%  \end{cases},
%  \end{align}where $max_j = \max_{i}{s_{i,j}}$ and $min_j = \min_{i}{s_{i,j}}$. In other words, the max and min values of all vectors in $\Set$ for a certain dimension $j$. Please note that $\Bell$ (resp. $\ell$) is a vector (resp. a scalar) in the case of box (resp. sphere) manifolds.

\subsection{Manifold-Matching GAN}
\begin{figure}[t]
\centering
\includegraphics[width=0.3\textwidth]{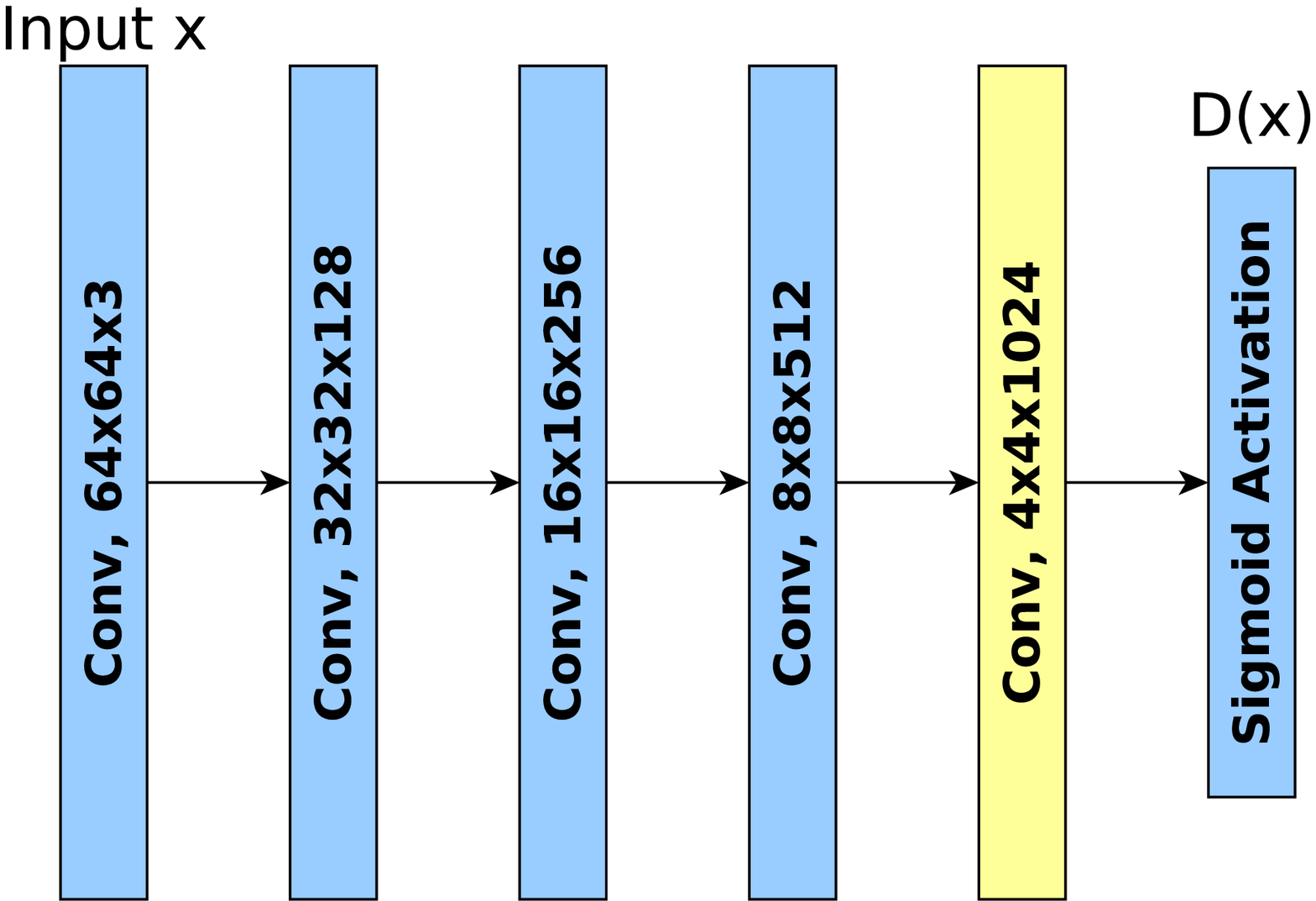}
% \subfigure[Generator $G$. Given a latent representation $z$, $G(z)$ is generated.]{\includegraphics[width=0.28\textwidth]
% {figures/generator.eps}}
\caption{Discriminator $D$. $D(\x)$ is computed after the last sigmoid activation. Vector representations of images are obtained by vectorizing the feature map marked in yellow.}\label{fig:dcgan}
% * <jramoniz@gmail.com> 2017-07-19T10:50:27.907Z:
% 
% > Given a set of inputs, their vector representations that can be achieved after vectoring the tensor marked in yellow in (a) are extracted to calculate a manifold representing them.
% 
% Maybe this should be:
% Given a set of inputs, their vector representations (which can be obtained by vectorizing the tensor marked in yellow in (a)) are extracted to calculate the manifold representing them.
% 
% ^.
\end{figure}

% The manifold-matching concept of the proposed manifold-matching GAN (MMGAN) can be implemented on various existing GAN models.
% As shown in Figure~\ref{fig:dcgan}, the architecture of the MMGAN consists of two convolutional neural networks, similar to that of DCGAN. In fact, IGAN also uses the neural network architecture similar to this figure.

We extract vector representations from the last layer of the discriminator $D$ as shown in Figure~\ref{fig:dcgan}, and the generator $G$ is trained using the following loss $\Loss_G$ based on the manifold-matching concept: 
\begin{small}\begin{equation}\begin{aligned}\label{eq:loss}
\Loss_G &= \| \mathbf{c_{M_\Y}} - \mathbf{c_{M_\X}} \|_2 + | \mathbf{\ell_{M_\Y}} - \mathbf{\ell_{M_\X}} |,
% \Bell &= \langle \frac{\sum_{s \in \Set}\abs{\mathbf{c}_0 - \mathbf{s}_0}}{n}, \frac{\sum_{s \in \Set}\abs{\mathbf{c}_1 - \mathbf{s}_1}}{n}, \cdots, \frac{\sum_{s \in \Set}\abs{\mathbf{c}_d - \mathbf{s}_d}}{n} \rangle,
%|\ell_{M_\Y} - \ell_{M_\X} |,
\end{aligned}\end{equation}\end{small}where $M_\Y$ is a sphere manifold representing real samples, and $M_\X$ is a sphere manifold of generated samples. 

In this loss $\Loss_G$, any space mapping is not used. In the next section, we will describe the loss $\Loss_G^{\K}$ when a space mapping is used. The above loss function requires that the two centroids be the same and that the difference between their radii be zero, which corresponds to the manifold-matching condition in Definition~\ref{def:mm}.

The discriminator $D$ is trained using the following loss:
\begin{small}\begin{equation}\begin{aligned}\label{eq:loss}
\Loss_D &= \Loss_{orig} - \Loss_G,
% \Bell &= \langle \frac{\sum_{s \in \Set}\abs{\mathbf{c}_0 - \mathbf{s}_0}}{n}, \frac{\sum_{s \in \Set}\abs{\mathbf{c}_1 - \mathbf{s}_1}}{n}, \cdots, \frac{\sum_{s \in \Set}\abs{\mathbf{c}_d - \mathbf{s}_d}}{n} \rangle,
%|\ell_{M_\Y} - \ell_{M_\X} |,
\end{aligned}\end{equation}\end{small}where $\Loss_{orig}$ is the original GAN loss shown in Equation~\ref{eq:gan} and $- \Loss_G$ is the adversarial version of the generator loss. In other words, the proposed MMGAN performs an adversarial game for manifold-matching.

The training of $D$ is also crucial since we extract vector representations from it. As the discriminator $D$ gets improved, the vector representations also become more accurate. As a result, the training process of MMGAN becomes more reliable. For this reason, we prefer a strong discriminator that is able to produce accurate vector representations to distinguish real and generated images. Recall that the generator is not well trained if the discriminator is too strong in the original GAN~\cite{DBLP:journals/corr/ArjovskyB17}. However, our MMGAN is free from the critical problem. This is one of our main contributions in this paper.

% Given a generator $G$ trained using $\Loss_G$, there is no difference in the form of the optimal discriminator since we use the original loss function to train the discriminator. Thus, the optimal discriminator $D^*_G$ has the same form as in Equation~\eqref{eq:optD}.

% \begin{proposition}
% In MMGAN, the optimal discriminator $D^*_G(\x)$ given a generator $G$ is the same as in Equation~\eqref{eq:optD}, because we use the original loss function to train the discriminator.
% \end{proposition}

% \begin{proposition}
% In MMGAN, the optimal generator $G^*_D(z)$ given a discriminator $D$ minimizes the loss value $\Loss_G$ down to 0. At this state, two manifolds representing real and generated samples becomes identical, which implies that $p_{data} = p_g$.
% % * <jramoniz@gmail.com> 2017-07-19T10:57:40.212Z:
% % 
% % > At this state, two manifolds representing real and generated samples are identical, i.e., $p_{data} = p_g$, because two manifolds are the same.
% % 
% % Maybe:
% % At this state, two manifolds representing real and generated samples are identical, which implies that $p_{data} = p_g$ because the manifolds representing them are the same.
% % Or more concisely:
% % At this state, two manifolds representing real and generated samples are identical, which implies that $p_{data} = p_g$.
% % 
% % ^.
% \end{proposition}

\section{Reproducing Kernel Hilbert Space\footnotesize{ (RKHS)}}\label{sec:kernel}

\begin{figure*}[ht]
\begin{small}\begin{center}\begin{equation}\begin{aligned}
\| \varphi(\mathbf{c}) - \varphi(\mathbf{s_i}) \|^2_2 &= \K(\mathbf{c},\mathbf{c}) -2\K(\mathbf{c},\mathbf{s_i}) + \K(\mathbf{s_i},\mathbf{s_i}) = \ell_{\K}, \\
\label{eq:kr}\ell_{\K} = \frac{\sum_{s_i \in \Set}\| \varphi(\mathbf{c}) - \varphi(\mathbf{s_i}) \|_2^2}{n} &= \frac{\sum_{s_i \in \Set}\K(\mathbf{c},\mathbf{c}) -2\K(\mathbf{c},\mathbf{s_i}) + \K(\mathbf{s_i},\mathbf{s_i})}{n},
\end{aligned}\end{equation}\end{center}\end{small}where $\varphi$ is a mapping from the original space to the Hilbert space and $\K$ is a kernel induced by the mapping $\varphi$. Thus, $\ell_{\K}$ stands for the radius after the space mapping.

\begin{footnotesize}
\begin{equation}\begin{aligned}\label{eq:kernelloss}
\Loss_G^{\K} &= \| \varphi(\mathbf{c_\Y}) - \varphi(\mathbf{c_\X}) \|^2_2 + \alpha\cdot\abs{\ell_{\K,M_\Y} - \ell_{\K,M_\X}} \\
&= \K(\mathbf{c_\Y},\mathbf{c_\Y}) -2\K(\mathbf{c_\Y},\mathbf{c_\X}) + \K(\mathbf{c_\X},\mathbf{c_\X}) \\&+ \alpha\cdot\abs{\frac{\sum_{y_i \in \Y}\K(\mathbf{c_\Y},\mathbf{c_\Y}) -2\K(\mathbf{c_\Y},\mathbf{y_i}) + \K(\mathbf{y_i},\mathbf{y_i})}{n} - \frac{\sum_{x_i \in \X}\K(\mathbf{c_\X},\mathbf{c_\X}) -2\K(\mathbf{c_\X},\mathbf{x_i}) + \K(\mathbf{x_i},\mathbf{x_i})}{n}},
\end{aligned}\end{equation}\end{footnotesize} where $\ell_{\K,M_\X}$ and $\ell_{\K,M_\Y}$ can be calculated with Equation~\eqref{eq:kr}.\end{figure*}

Reproducing Kernel Hilbert Space (RKHS) is known to be a flexible approach to represent manifolds. In RKHS, a kernel $\K$ is used to lay the proposed sphere manifold representing a set $\Set = \{\mathbf{s_1}, \mathbf{s_2}, \cdots, \mathbf{s_n}\}$ in a Hilbert space as Equation~\eqref{eq:kr}.

There exist many different kernels, such as linear, Gaussian, and polynomial. In particular, the linear kernel is defined as $\K(\mathbf{a},\mathbf{b})=\mathbf{a}^\intercal \mathbf{b}$, and it reduces to mapping to the original space (i.e., $\varphi(s)=s$ with the linear kernel).

After applying the kernel trick, the loss function to train the generator can be rewritten as $\Loss_G^{\K}$ in Equation~\eqref{eq:kernelloss}. With this kernel trick, we can easily apply various kernels without any explicit space mapping. We skip the description about each kernel function since they are already well-known.

\section{Correlation Regularization}
% * <jramoniz@gmail.com> 2017-07-19T11:08:22.047Z:
% 
% > Covariance Regularization
% 
% I really liked the idea of covariance regularization as well. It is possible that this idea is applicable to other GAN models too. Maybe we could have a tiny section in the experiments, showing results with and without covariance regularization, and/or with different \beta values? And mention that, based on the diversity we have observed here, that we have strong reason to believe that it might be applicable to other GAN models as well, and:
% 
% 1. say that we do not explore that line of reasoning further, since it is out of scope of the specific context of the paper
% 2. show some backing experiments in an appendix
% 
% ^.
One key factor of successful training in MMGAN is the diversity of generated samples --- more precisely, the diversity in vector representations of generated samples. In an epoch, if all generated samples are similar to each other, then its manifold will be small and all the samples' vector representations will gather around its centroid, also known as mode collapsing, whereas randomly selected training image samples are usually much more diverse than the generated ones. Even when mode collapsing occurs, two centroids can be similar (i.e., $\| \varphi(\mathbf{c_\Y}) - \varphi(\mathbf{c_\X}) \|^2_2 \approx 0$), which is not preferred.
% * <jramoniz@gmail.com> 2017-07-19T11:02:13.544Z:
% 
% > As long as two centroids are similar (i.e. $\| \mathbf{c_{M_\Y}} - \mathbf{c_{M_\X}} \|_{1\textrm{ or }2} \approx 0$), half of $\Loss_G$ will be close to 0.
% 
% I didn't quite follow this point. Could you please help me?
% 
% ^.

To improve the training procedure, we use the following regularization that promotes the diversity in generated samples at the stage of training the generator $G$. Given a discriminator $D$, the generator is trained so that it can generate more diverse samples.

That is, given a set of vector representations of samples (e.g., $\Set = \{\mathbf{s_1}, \mathbf{s_2}, \cdots, \mathbf{s_n}\}$), we first calculate their $n \times n$ correlation matrix (i.e., normalized covariance matrix) $A_\Set$.  An element $a_{i,j}$ of $A_\Set$ is a correlation value of vector representations between the $i$-th and the $j$-th samples of $\Set$. Then, our proposed regularization term is defined as 
\begin{equation}\label{eq:cov}
\R_G = \| A - A_\Set \|_{\F},
\end{equation}where $A$ is an identity matrix, which is the desired correlation matrix and means vector representations are independent from each other; and $\|\cdot\|_{\F}$ is the Frobenius matrix norm. The identity correlation matrix $A$ represents the case that each sample's vector representation is completely independent and the samples in $\Set$ are diverse.

\begin{theorem}
If $A_\Set = A$ in Equation~\eqref{eq:cov}, then samples in $\Set$ are all different.
\end{theorem}
\begin{proof}
First, assume that i) no different images have the exactly same vector representation and ii) similar images have similar vector representations. These two assumptions are most likely to be the case if the discriminator is enough trained. In fact, the second assumption is a property of DCGAN.

If mode collapse happens, similar samples will be generated regardless of the input vector $z$, i.e. $G(z_1) \approx G(z_2)$, where $z_1 \neq z_2$. Thus, their vector representations will show high correlation (by the second assumption). In other words, $a_{i,j}$ of $A_S$, where $i \neq j$, will be close to 1. In the optimal form of the correlation matrix $A$, $a_{i,j} = 0$, which means $G(z_i) \neq G(z_j)$ (i.e., no mode collapse).

Therefore, if $\R_G = 0$, then there is no mode-collapse.
\end{proof}

Finally, in MMGAN, neural networks are trained with the following loss functions: 
\begin{align}\begin{split}\label{eq:mmganloss}
\Loss_G^{final} &= 
\begin{cases}
\Loss_G + \beta\cdot \R_G, & \textrm{without any kernel} \\
% \| \mathbf{c_\Y} - \mathbf{c_\X} \|_{1\textrm{ or }2} + \alpha\cdot\| \Bell_\mathbf{\Y} - \Bell_\mathbf{\X} \|_{1\textrm{ or }2} + \beta\cdot \R_G, & \textrm{if box}\\
\Loss_G^{\K}  + \beta\cdot \R_G, & \textrm{with a kernel $\K$}
\end{cases}\\
\Loss_D^{final} &= \Loss_{orig} - \Loss_G^{final}
\end{split}\end{align} where $\alpha$ and $\beta$ are weight parameters and $\Loss_G$ and $\Loss_G^{\K}$ are defined in Equations~\eqref{eq:loss} and~\eqref{eq:kernelloss}, respectively.

\section{Training Algorithm}
Algorithm~\ref{alg:mmgan} describes the training procedures of MMGAN based on the proposed loss functions. Basically, the overall training procedure is the same as the original GAN, but the main difference is that we maintain two manifolds representing $p_{data}$ and $p_g$ by using the exponentially-weighted moving average calculation of $\mathbf{c_\Y}$, $\mathbf{\Bell_\Y}$, $\mathbf{c_\X}$, and $\mathbf{\Bell_\X}$. With this moving average calculation, we aim to find statistically meaningful manifolds because we consider all the previous samples. \textbf{In particular, we set $\delta$ as $0.9$ or greater so that we can give a large weight on the cumulative manifold information rather than the manifold representing only the mini-batch samples}.

\begin{algorithm2e}[!t]
\footnotesize
\DontPrintSemicolon
\hrule
\KwIn{real samples: $\{x_1, x_2, \cdots\} \sim p(x)$}
\KwOut{a generative model $G$}
\hrule
%{\color{blue}//Initialize a generator $G$ and discriminator $D$}\;
$G \gets$ a generative neural network\;
$D \gets$ a discriminator neural network\;
\While{until convergence} {
Create a mini-batch of real samples $Y_{mini}= \{y_1,\cdots,y_n\}$\;
Create a set of generated samples $Z_{mini} = \{z_1,\cdots,z_n\}$ and their generated images $X_{mini} = \{G(z_0),\cdots,G(z_n)\}$\;
Train the discriminator $D$ with $\Loss_D^{final}$\;\label{alg:dis2}
\tcc{Moving average update of the centroid and the radius for $p_{data}$}
% * <jramoniz@gmail.com> 2017-07-19T13:28:57.402Z:
% 
% Would describing c_{\Y_{mini}} etc here help improve clarity, or are they self-explanatory?
% 
% ^.
$\mathbf{c_\Y} = \delta \times \mathbf{c_\Y} + (1-\delta) \times \mathbf{c_{\Y_{mini}}}$\;\label{alg:mva1}
$\mathbf{\Bell_\Y} = \delta \times \mathbf{\Bell_\Y} + (1-\delta) \times \mathbf{\Bell_{\Y_{mini}}}$\;
\tcc{Moving average update of the centroid and radius for $p_{g}$}
$\mathbf{c_\X} = \delta \times \mathbf{c_\X} + (1-\delta) \times \mathbf{c_{\X_{mini}}}$\;
$\mathbf{\Bell_{\X}} = \delta \times \mathbf{\Bell_{\X}} + (1-\delta) \times \mathbf{\Bell_{\X_{mini}}}$\;\label{alg:mva2}
Train the generator $G$ with $\Loss_G^{final}$\;

}
\Return $G$\;
\hrule
\caption{\strut Training algorithm of MMGANs. $\Y_{mini}$ (resp. $\X_{mini}$) is the set of vector representations of a set of real samples $Y_{mini}$ (resp. generated samples $X_{mini}$)\label{alg:mmgan}}
\end{algorithm2e}

The runtime of Algorithm~\ref{alg:mmgan} remains similar to it of IGAN except the correlation regularization that additionally incurs a couple of $n \times d$ matrix multiplications. However, this is not significant because existing deep learning platforms can quickly calculate.

\section{Experiments}
% In this section, we describe our experimental setups and results.\footnote{The source code and datasets of all the experiments used in this paper will be released upon acceptance of the paper.} We forked the DCGAN\footnote{\url{github.com/carpedm20/DCGAN-tensorflow}} and IGAN\footnote{\url{github.com/openai/improved-gan}} source codes and implemented our MMGAN.
In this section, we describe our evaluation environments and results. The source code and datasets of all the experiments used in this paper will be available at (\url{github.com/npark/MMGAN}) upon publication. Now we are cleaning them. We forked the DCGAN\footnote{\url{github.com/carpedm20/DCGAN-tensorflow}} and IGAN\footnote{\url{github.com/openai/improved-gan}} source codes and implemented our MMGAN on our own.

Note that there exists no solid metric to quantitatively evaluate the quality of generated images. While the inception score~\cite{DBLP:journals/corr/SalimansGZCRC16} can be used for some images, we generally rely on humans' visual perception to evaluate the quality of generated samples.

\subsection{Experimental Setup}\label{sec:env}
We use three image datasets, MNIST~\cite{lecun-mnisthandwrittendigit-2010}, CelebA~\cite{liu2015faceattributes}, and CIFAR-10~\cite{cifar}. The MNIST database, one of the most popular evaluation datasets, contains 60,000 hand-written digit images. CelebA has over 360,000 celebrity face images with a large variety of facial poses. CIFAR-10 consists of 52,000 images from 10 different object classes. While the images in CelebA and CIFAR-10 have RGB channels, MNIST images are grey-scale ones. 

We rely on human visual perception to evaluate MNIST and CelebA since there is no quantitative evaluation metric. To this end, we created a website for a user study,\footnote{The website is available at \url{gantest.herokuapp.com}. Using the passkey 0DEE2A, one can click downvote icons for an image if  s/he thinks it as fake and submits the results. Results with the passkey will not be counted in the results.} where given randomly chosen images, participants are asked to distinguish between real and fake images. More than 30 people from three different continents participated in the study. Each participant performed the same number of trials for each generative method to avoid any potential biases. We demonstrate how many fake images are correctly identified as fake for each method.

For CIFAR-10, we primarily use the \emph{inception score} defined in~\cite{DBLP:journals/corr/SalimansGZCRC16}. It uses the inception model~\cite{DBLP:journals/corr/SzegedyVISW15} to detect objects in generated images. The inception score is designed to give high scores if \emph{various} and \emph{high-quality} objects are recognized in the generated samples.

We used a cluster of machines running Linux with Xeon CPU, 32GB RAM, and K80 GPUs for all our experiments, and the Tensorflow \cite{abadi2016tensorflow} and Theano \cite{bergstra2011theano} deep neural network libraries.

\subsection{Experimental Results}
In Figure~\ref{fig:cifar}, we show generation examples for some random mini-batches at the last epoch. We did not cherry-pick them. It is very noticeable that almost all images are visually recongnizable.

\begin{table}[t]
\scriptsize
\centering
\caption{The percentage of samples selected as fake in the user study. Lower values are preferred.}
\label{user}
\begin{tabular}{|c|c|c|c|c|c|}
\hline
Dataset & Real & DCGAN & \begin{tabular}[c]{@{}c@{}}MMGAN\\ on DCGAN\\(RBF Kernel)\end{tabular} & \begin{tabular}[c]{@{}c@{}}MMGAN\\ on DCGAN\\(EXP Kernel)\end{tabular} & \begin{tabular}[c]{@{}c@{}}MMGAN\\ on DCGAN\\(No $\R_G$)\end{tabular}\\ \hline
MNIST & 14\% & 38.2\%  & \textbf{32.4\%} & 35.4\% & 45.5\% \\ \hline
CelebA & 4\%  & 75.3\% & 77\%  & \textbf{73.3\%} & 76.6\% \\ \hline
\end{tabular}
\end{table}

\begin{figure*}[!htb]
\centering
\subfigure[DCGAN]{\includegraphics[width=0.49\textwidth]{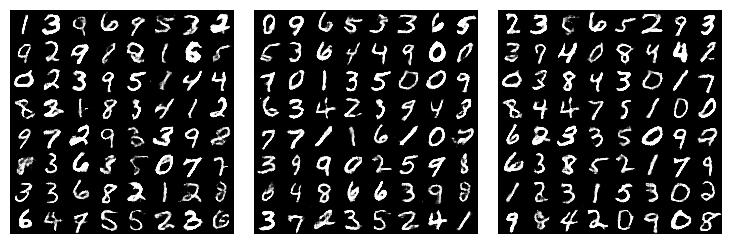}}
% \subfigure[DCGAN samples reproduced by us]{\includegraphics[width=0.7\textwidth]
% {figures/test_samples_mnist_orig.jpg}}\\
% \subfigure[MMGAN (No Kernel, $\alpha=0.5$, $\beta=0.1$)]{\includegraphics[width=0.32\textwidth]
% {figures/test_samples_mnist_R0_5_Cov0_1_O0_C0_5_none_64.jpg}}
\subfigure[MMGAN (RBF, $\alpha=1.0$, $\beta=1.0$)]{\includegraphics[width=0.49\textwidth]
{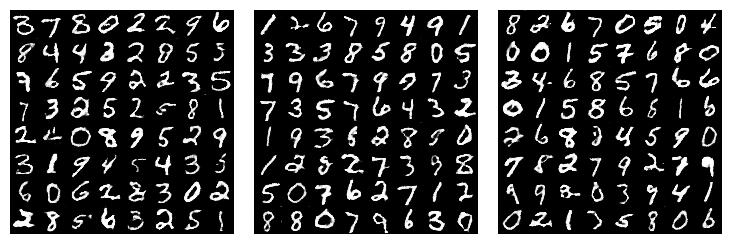}}

\subfigure[DCGAN]{\includegraphics[width=0.49\textwidth]
{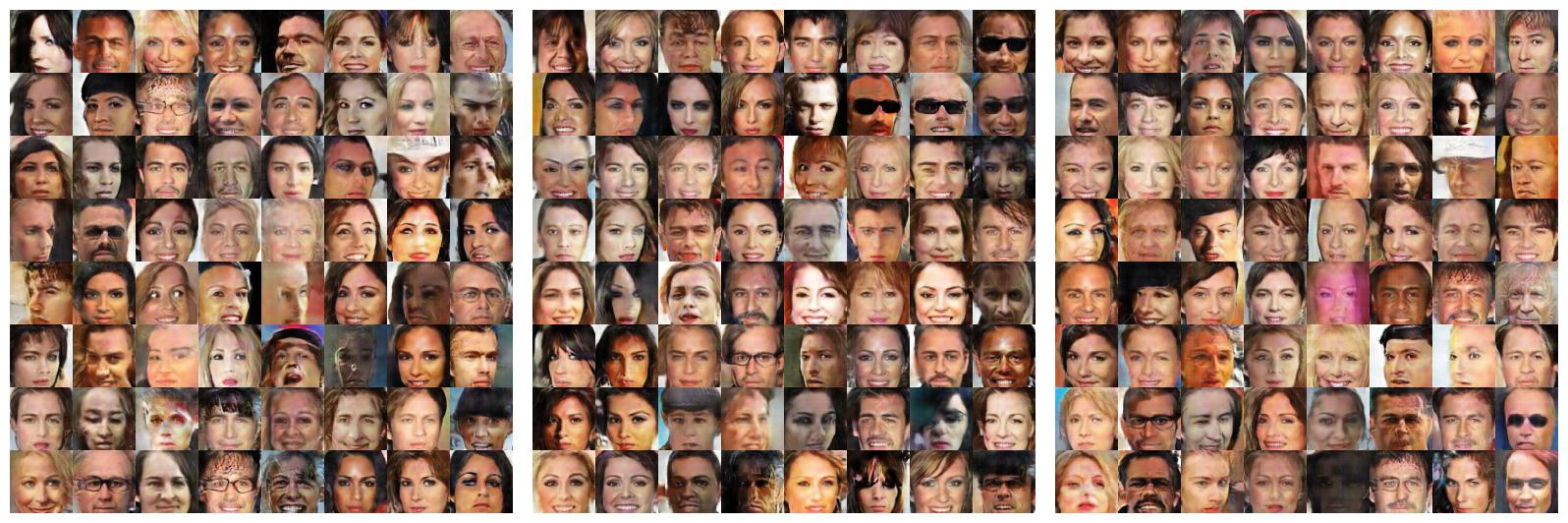}}
% \subfigure[MMGAN (No Kernel, $\alpha=1.0$, $\beta=1.0$)]{\includegraphics[width=0.32\textwidth]
% {figures/test_samples_celebA_R1_Cov1_O0_C1_none_64.jpg}}
\subfigure[MMGAN (EXP, $\alpha=1.0$, $\beta=1.0$)]{\includegraphics[width=0.49\textwidth]
{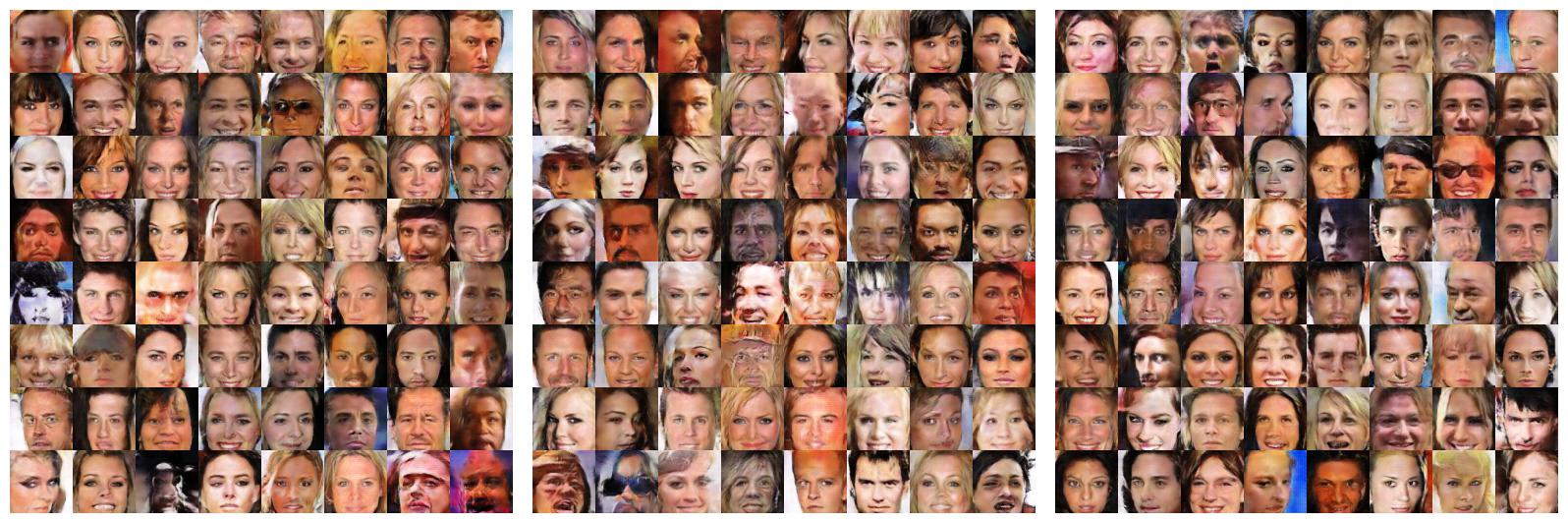}}

% \subfigure[MMGAN (No Kernel, $\alpha=1.0$, $\beta=1.0$)]{\includegraphics[width=0.25\textwidth]
% {figures/bd_cifar10_1_00_1_00_0_10_0_99_none_sample_feature_match.png}}
\subfigure[MMGAN (RBF, $\alpha=1.0$, $\beta=0.1$)]{\includegraphics[width=0.25\textwidth]
{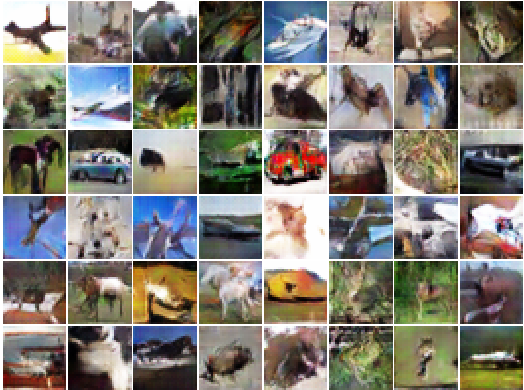}}
\subfigure[MMGAN (EXP, $\alpha=1.0$, $\beta=0.1$)]{\includegraphics[width=0.25\textwidth]
{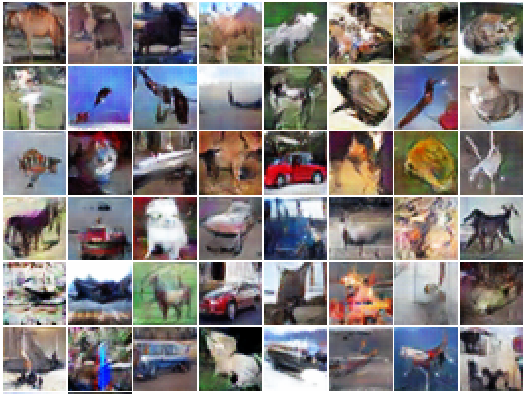}}
\qquad
\subfigure[Inception score for CIFAR-10 (class-unsupervised training)]{
\begin{footnotesize}\begin{tabular}[b]{|c|c|c|c|}
\hline
GAN model & DCGAN & \begin{tabular}[c]{@{}c@{}}MMGAN\\ on DCGAN\\(RBF Kernel)\end{tabular} & AGE \\ \hline
Inception score &  6.8  & $6.96 \pm 0.06$  &  5.6 \\ \hline
GAN model & IGAN & \begin{tabular}[c]{@{}c@{}}MMGAN\\ on IGAN\\(EXP Kernel)\end{tabular} & $\alpha$-GAN \\ \hline
Inception score & 6.86 $\pm$ 0.06 & 7.8 $\pm$ 0.06 & 7.0 \\ \hline
\end{tabular}\end{footnotesize}}

\caption{Generated samples for (a)-(b) MNIST, (c)-(d) CelebA, (e)-(f) CIFAR-10, and (g) inception score for CIFAR-10. We did not cherry-pick generated images but show random mini-batches. Note that almost all generated images show high quality.}\label{fig:cifar}
\end{figure*}

\subsubsection{MNIST}
Figure~\ref{fig:cifar} (a) and (b) shows those image samples generated by DCGAN and MMGAN (implemented after modifying DCGAN). Figure~\ref{fig:cifar}(a) are generated examples available at the official github repository of DCGAN. Many samples are properly recognizable by humans, while a few samples are of poor quality. Figures~\ref{fig:cifar}(b) shows images generated by MMGAN with the RBK kernel.

Our user study results in Table~\ref{user} show that MMGAN received a less number of downvotes than DCGAN. This indicates that the participants considered  MMGAN-generated images more realistic than DCGAN-generated images. Note also that without the correlation regularization, MMGAN does not outperform DCGAN, which implies that the proposed manifold-matching becomes more robust when generated samples are diverse.

Interestingly, 14\% of real images of MNIST also received downvotes. This means our high-quality generated images make participants confused and they committed non-trivial mistakes in the user study.

% Samples in Figure~\ref{fig:mnist} (b) are generated by our execution of DCGAN. We trained DCGAN for 25 epochs, which is the default configuration of the DCGAN implementation. The quality of our samples is slightly worse than that of the official samples.

% MMGAN is also trained for 25 epochs. These images are much better in quality than Figure~\ref{fig:mnist} (b). In comparison with Figure~\ref{fig:mnist} (a), MMGAN shows better quality in a few cases. Some digits in Figure~\ref{fig:mnist} (a) are not clear, but there are far fewer unclear digits in Figure~\ref{fig:mnist} (c) and (d).

\subsubsection{CelebA}
% \begin{table*}[!htb]
% \scriptsize
% \centering
% \caption{Inception scores of various GANs for CIFAR-10 dataset. The original images of CIFAR-10 has the inception score of $11.24 \pm 0.12$ for its mean score and the standard deviation.}
% % The ``reported'' corresponds to the scores shown in their original papers; the ``reproduced'' corresponds to those results generated in our experiments with the recommended parameter setting from the original papers.}
% \label{inception}
% \begin{tabular}{|c|c|c|c|c|c|c|c|}
% \hline
% GAN model & DCGAN & \begin{tabular}[c]{@{}c@{}}MMGAN\\ on DCGAN\\(RBF Kernel)\end{tabular} & WGAN-GP & IGAN & \begin{tabular}[c]{@{}c@{}}MMGAN\\ on IGAN\\(EXP Kernel)\end{tabular} & AGE & $\alpha$-GAN \\ \hline
% Inception score &  6.8  & $6.96 \pm 0.06$  &  7.8 $\pm$ 0.1 & 8.09 $\pm$ 0.07 & 7.85 $\pm$ 0.06 & 5.6 & 7.0 \\ \hline
% \end{tabular}
% \end{table*}
Figure~\ref{fig:cifar} (c) and (d) show the comparison of CelebA samples generated by DCGAN and MMGAN (implemented after modifying DCGAN). The user study results show that many images generated by DCGAN and MMGAN were commonly identified as fake by participants. However, MMGAN slightly outperforms DCGAN. It is more challenging to generate high-quality images for CelebA than MNIST and CIFAR-10 because CelebA does not have labels and all facial images look similar whereas MNIST and CIFAR-10 contain diverse hard-written digits and objects. 

%We manually counted the number of incomplete samples among all 6,400 generated samples in each method. Thus, we checked 12,800 images (for DCGAN and MMGAN) in total with 10 human evaluators. Each evaluator is required to mark incomplete facial images with red circles and all marked results are in our uploaded file. Our guideline is to ensure that the eyes, nose \emph{and} mouth are all recognizable, but each evaluator may have different criteria when performing the task. Thus, we distributed the same number of images for DCGAN and MMGAN to an evaluator to address the possible fairness issue.

\subsubsection{CIFAR-10}
For CIFAR-10, we compare $\alpha$-GAN~\cite{2017arXiv170604987R}, DCGAN, IGAN,and MMGAN (implemented after modifying both DCGAN and IGAN). IGAN is the state-of-the-art method for this dataset~\cite{DBLP:journals/corr/Goodfellow17}. We use the inception score (rather than visual evaluations). $\alpha$-GAN is based on auto-encoders. $\alpha$-GAN was very recently released and outperforms other auto-encoder-based models such as~\cite{DBLP:journals/corr/UlyanovVL17a} by a considerable margin.

Note that MMGAN performs better than DCGAN and IGAN, as shown in the table of Figure~\ref{fig:cifar} (g). This highlights the superiority of MMGAN in comparison with conventional GAN training procedures. IGAN reported the inception score of 6.86 and our MMGAN on top of it marked a much better score. Samples generated by MMGAN is in Figure~\ref{fig:cifar} (e) and (f).

\section{Conclusion}
In this paper, we proposed a novel GAN model based on manifold-matching (MMGAN). The proposed MMGAN differs from other existing GAN models, especially with respect to the training loss function of the generator. Instead of the predictions $D(x)$ and $D(G(z))$ made by the discriminator, MMGAN extracts vector representations of real and fake images from the last layer of the discriminator and calculate an approximated manifold for each of the set of real images and that of fake ones. In order to better train the generator, MMGAN requires that the two manifolds (corresponding to the real and generated data) match. MMGANs hence benefit from a more accurate discriminator, which enables them to make more accurate vector representations of images. Our experiments demonstrates that MMGAN shows comparable or  better performance in comparison with other popular GAN models for three image datasets.

% \acks{Our research was funded by Electronics and Telecommunications Research Institute (ETRI) grant PF17-0567.}

\bibliographystyle{abbrv}
\bibliography{aaai}

\begin{thebibliography}{10}

\bibitem{abadi2016tensorflow}
M.~Abadi, A.~Agarwal, P.~Barham, E.~Brevdo, Z.~Chen, C.~Citro, G.~S. Corrado,
  A.~Davis, J.~Dean, M.~Devin, et~al.
\newblock Tensorflow: Large-scale machine learning on heterogeneous distributed
  systems.
\newblock {\em arXiv preprint arXiv:1603.04467}, 2016.

\bibitem{DBLP:journals/corr/ArjovskyB17}
M.~Arjovsky and L.~Bottou.
\newblock Towards principled methods for training generative adversarial
  networks.
\newblock {\em CoRR}, abs/1701.04862, 2017.

\bibitem{arjovsky2017wasserstein}
M.~Arjovsky, S.~Chintala, and L.~Bottou.
\newblock Wasserstein gan.
\newblock {\em arXiv preprint arXiv:1701.07875}, 2017.

\bibitem{DBLP:journals/corr/BaoCWLH17}
J.~Bao, D.~Chen, F.~Wen, H.~Li, and G.~Hua.
\newblock {CVAE-GAN:} fine-grained image generation through asymmetric
  training.
\newblock {\em CoRR}, abs/1703.10155, 2017.

\bibitem{DBLP:journals/corr/BellemareDDMLHM17}
M.~G. Bellemare, I.~Danihelka, W.~Dabney, S.~Mohamed, B.~Lakshminarayanan,
  S.~Hoyer, and R.~Munos.
\newblock The cramer distance as a solution to biased wasserstein gradients.
\newblock {\em CoRR}, abs/1705.10743, 2017.

\bibitem{bergstra2011theano}
J.~Bergstra, O.~Breuleux, P.~Lamblin, R.~Pascanu, O.~Delalleau, G.~Desjardins,
  I.~Goodfellow, A.~Bergeron, Y.~Bengio, and P.~Kaelbling.
\newblock Theano: Deep learning on gpus with python.
\newblock 2011.

\bibitem{goodfellow2014generative}
I.~Goodfellow, J.~Pouget-Abadie, M.~Mirza, B.~Xu, D.~Warde-Farley, S.~Ozair,
  A.~Courville, and Y.~Bengio.
\newblock Generative adversarial nets.
\newblock In {\em Advances in neural information processing systems}, pages
  2672--2680, 2014.

\bibitem{DBLP:journals/corr/Goodfellow17}
I.~J. Goodfellow.
\newblock {NIPS} 2016 tutorial: Generative adversarial networks.
\newblock {\em CoRR}, abs/1701.00160, 2017.

\bibitem{hjelm2017boundary}
R.~D. Hjelm, A.~P. Jacob, T.~Che, K.~Cho, and Y.~Bengio.
\newblock Boundary-seeking generative adversarial networks.
\newblock {\em arXiv preprint arXiv:1702.08431}, 2017.

\bibitem{hofmann2008}
T.~Hofmann, B.~Sch\"{o}lkopf, and A.~J. Smola.
\newblock Kernel methods in machine learning.
\newblock {\em Annals of Statistics}, 36(3):1171--1220, 2008.

\bibitem{2016arXiv160603439K}
T.~{Kim} and Y.~{Bengio}.
\newblock {Deep Directed Generative Models with Energy-Based Probability
  Estimation}.
\newblock {\em ArXiv e-prints}, June 2016.

\bibitem{cifar}
A.~Krizhevsky, V.~Nair, and G.~Hinton.
\newblock Cifar-10 (canadian institute for advanced research).

\bibitem{lecun-mnisthandwrittendigit-2010}
Y.~LeCun and C.~Cortes.
\newblock {MNIST} handwritten digit database.
\newblock 2010.

\bibitem{liu2015faceattributes}
Z.~Liu, P.~Luo, X.~Wang, and X.~Tang.
\newblock Deep learning face attributes in the wild.
\newblock In {\em Proceedings of International Conference on Computer Vision
  (ICCV)}, 2015.

\bibitem{DBLP:journals/corr/MrouehSG17}
Y.~Mroueh, T.~Sercu, and V.~Goel.
\newblock Mcgan: Mean and covariance feature matching {GAN}.
\newblock {\em CoRR}, abs/1702.08398, 2017.

\bibitem{NIPS2010_3958}
H.~Narayanan and S.~Mitter.
\newblock Sample complexity of testing the manifold hypothesis.
\newblock In J.~D. Lafferty, C.~K.~I. Williams, J.~Shawe-Taylor, R.~S. Zemel,
  and A.~Culotta, editors, {\em Advances in Neural Information Processing
  Systems 23}, pages 1786--1794. Curran Associates, Inc., 2010.

\bibitem{radford2015unsupervised}
A.~Radford, L.~Metz, and S.~Chintala.
\newblock Unsupervised representation learning with deep convolutional
  generative adversarial networks.
\newblock {\em arXiv preprint arXiv:1511.06434}, 2015.

\bibitem{2017arXiv170604987R}
M.~{Rosca}, B.~{Lakshminarayanan}, D.~{Warde-Farley}, and S.~{Mohamed}.
\newblock {Variational Approaches for Auto-Encoding Generative Adversarial
  Networks}.
\newblock {\em ArXiv e-prints}, June 2017.

\bibitem{Roweis00nonlineardimensionality}
S.~T. Roweis and L.~K. Saul.
\newblock Nonlinear dimensionality reduction by locally linear embedding.
\newblock {\em SCIENCE}, 290:2323--2326, 2000.

\bibitem{DBLP:journals/corr/SalimansGZCRC16}
T.~Salimans, I.~J. Goodfellow, W.~Zaremba, V.~Cheung, A.~Radford, and X.~Chen.
\newblock Improved techniques for training gans.
\newblock {\em CoRR}, abs/1606.03498, 2016.

\bibitem{DBLP:journals/corr/SzegedyVISW15}
C.~Szegedy, V.~Vanhoucke, S.~Ioffe, J.~Shlens, and Z.~Wojna.
\newblock Rethinking the inception architecture for computer vision.
\newblock {\em CoRR}, abs/1512.00567, 2015.

\bibitem{DBLP:journals/corr/UlyanovVL17a}
D.~Ulyanov, A.~Vedaldi, and V.~S. Lempitsky.
\newblock Adversarial generator-encoder networks.
\newblock {\em CoRR}, abs/1704.02304, 2017.

\bibitem{2016arXiv160903126Z}
J.~{Zhao}, M.~{Mathieu}, and Y.~{LeCun}.
\newblock {Energy-based Generative Adversarial Network}.
\newblock {\em ArXiv e-prints}, Sept. 2016.

\end{thebibliography}
\end{document}